\documentclass[12pt]{colt2021} 
\usepackage{verbatim}

\newtheorem{assumption}[theorem]{Assumption}

\renewcommand*{\vec}{}

\usepackage{accents}

\newcommand{\argmin}{\mathop{\mathrm{argmin}}}

\def\cD{\mathcal{D}}

\def\cH{\mathcal{H}}

\def\cN{\mathcal{N}}
\def\cP{\mathcal{P}}

\def\cX{\mathcal{X}}

\def\bs{\ensuremath\boldsymbol}

\title{Dynamic Regret for Strongly Adaptive Methods and Optimality of Online KRR}
\usepackage{times}
\usepackage{jmlrutils}

\usepackage{mathtools}  
\mathtoolsset{showonlyrefs}

\coltauthor{%
 \Name{Dheeraj Baby}\Email{dheeraj@ucsb.edu}\\
 \addr{University of California, Santa Barbara}
 \AND
 \Name{Hilaf Hasson}\Email{hashilaf@amazon.com}\\ \addr{Amazon Research}
 \AND
 \Name{Yuyang Wang}\Email{yuyawang@amazon.com}\\
 \addr{Amazon Research}
}

\begin{document}

\maketitle

\begin{abstract}%
We consider the framework of non-stationary Online Convex Optimization where a learner seeks to control its \emph{dynamic regret} against an \emph{arbitrary} sequence of comparators. When the loss functions are strongly convex or exp-concave, we demonstrate that Strongly Adaptive (SA) algorithms can be viewed as a principled way of controlling dynamic regret in terms of \emph{path variation} $V_T$ \emph{of the comparator sequence}. Specifically, we show that SA algorithms enjoy $\tilde O(\sqrt{TV_T} \vee \log T)$ and $\tilde O(\sqrt{dTV_T} \vee d\log T)$ dynamic regret for strongly convex and exp-concave losses respectively \emph{without} apriori knowledge of $V_T$. The versatility of the principled approach is further demonstrated by the novel results in the setting of learning against bounded linear predictors and online regression with Gaussian kernels.

Under a related setting, the second component of the paper addresses an open question posed by \cite{KRR} that concerns online kernel regression with squared error losses. We derive a new lower bound on a certain \emph{penalized regret} which establishes the near minimax optimality of online Kernel Ridge Regression (KRR). Our lower bound can be viewed as an RKHS extension to the lower bound derived in \cite{vovk2001} for online linear regression in finite dimensions.
\end{abstract}


\section{Introduction} \label{sec:intro}
Online Convex Optimization (OCO) is a powerful learning paradigm for real-time decision making. It has been applied in many influential applications such as portfolio selection, time series forecasting, and online recommendation systems to cite a few \citep{hazan2007logregret, koolen2015minimax,hazan2016introduction}. The OCO problem is modelled as an iterative game between a learner and an adversary that proceeds for $T$ rounds as follows. At each time step $t$, the learner chooses a point $\bs x_t$ in a convex decision set $\cD$. Then the adversary reveals a convex loss function $f_t:\cD \rightarrow \mathbb{R}$. The most common way of measuring the performance of a learner is via its static regret, $R^s_T(z):=\sum_{t=1}^T (f_t(\bs x_t) -  f_t(\bs z))$, where $\bs z$ is termed as a fixed \emph{comparator} in hindsight which can be any point in $\cD$. For example $\bs z$ can be chosen as $\argmin_{\bs x \in \cD} \sum_{t=1}^T f_t(\bs x)$ with the knowledge of the entire sequence of loss functions. Learning is said to happen whenever the regret grows sub-linearly w.r.t. $T$. However, in the case of non-stationary environments such as stock market, one is often interested in matching the performance of a sequence of decisions in hindsight. In such circumstances, the notion of static regret fails to assess the performance of the learner.  To better capture the non-stationarity, \cite{zinkevich2003online} introduces the notion of \emph{dynamic regret}:

\begin{align}
    R_T(\bs z_1,\ldots,\bs z_T) := \sum_{t=1}^T (f_t(\bs x_t) - f_t(\bs z_t)), \label{eq:d-regret}
\end{align}
where $\bs z_1,\ldots,\bs z_T$ is \emph{any} sequence of comparators in $\cD$. The degree of non-stationarity present in the comparator sequence is measured using the path variational defined as
\begin{equation}
    V_T(\bs z_1,\ldots,\bs z_T) := \sum_{t=2}^{T} \|\bs z_t - \bs z_{t-1} \|, \label{eq:path-var}
\end{equation}
where $\| \cdot \|$ is the Euclidean norm. In what follows, we drop the arguments and represent the variation by $V_T$ for brevity. The dynamic regret bounds are usually expressed as a function of $T$ and $V_T$.


It is known that with convex loss functions the optimal dynamic regret is $O(\sqrt{T(1+V_T)})$ \citep{zhang2018adaptive} which improves to $O(\sqrt{dTV_T} \vee d\log T)$ \citep{yuan2019dynamic}, where $d$ is the dimensionality of $\mathcal{D}$ and $(a \vee b) = \max \{a,b \}$,  with additional curvature properties such as exp-concavity. 

A parallel line of research \citep{hazan2007adaptive,daniely2015strongly,koolen2016specialist} focus on developing algorithms whose static regret is controlled in any time interval. Specifically, \cite{daniely2015strongly} develops the notion of Strongly Adaptive (SA) algorithms defined as:

\begin{definition} \citep{daniely2015strongly} \label{def:sa}
Let $[T]:=\{1,\ldots,T \}$. An algorithm is said to be Strongly Adaptive  if for every continuous interval $I \subseteq [T]$, the static regret incurred by the algorithm is $O(\text{poly}(\log T) R^*(|I|))$, where $R^*(|I|)$ is the value of minimax static regret incurred in an interval of length $|I|$.
\end{definition}


\cite{zhang2018dynamic} shows that SA algorithms incur dynamic regret of $\tilde O(T^{2/3}C_T^{1/3})$ \footnote{$\tilde O(\cdot)$ hides polynomial factors of $\log T$.} for convex losses and $\tilde O(\sqrt{TC_T})$ for strongly convex losses, where 
\begin{align}
    C_T
    &:= \sum_{t=2}^{T} \sup_{\bs x \in \cD} |f_t(\bs x) - f_{t-1}(\bs x)| \label{eq:c_t}
\end{align}
to captures the non-stationarity of the problem in terms of the degree to which the sequence of losses changes over time. They further show that both results are optimal modulo poly logarithmic factors of $T$. For exp-concave losses, they derive a regret bound of $\tilde O(\sqrt{dTC_T})$. However, in \cite{zhang2018dynamic}, a question that was left open is whether it is possible to derive dynamic regret rates for SA methods that depend on the variational $V_T$. 

In this paper, we answer this affirmatively for strongly convex (Theorem \ref{thm:sc}) and exp-concave (Theorem \ref{thm:ec}) losses. Specifically, we show that for SA methods,
\begin{align}
    R_T(\bs z_1,\ldots,\bs z_T)
    &=\tilde O(\sqrt{TV_T(\bs z_1,\ldots,\bs z_T)} \vee \log T), \quad \text{(for strongly convex losses)}
\end{align}
and
\begin{align}
    R_T(\bs z_1,\ldots,\bs z_T)
    &=\tilde O(\sqrt{dTV_T(\bs z_1,\ldots,\bs z_T)} \vee d\log T). \quad \text{(for exp-concave losses)}
\end{align}

This result immediately implies that SA algorithms can be seen as a unifying framework that allows one to control dynamic regret under different variationals ($C_T$ and $V_T$) \emph{simultaneously} whenever losses have curvature properties. Though this dynamic regret is attained by \cite{yuan2019dynamic} (without $\log T$ factors) by fundamentally different algorithms, our proof techniques are much simpler and shorter. Further, we demonstrate the versatility of this perspective by deriving new dynamic regret rates in various other interesting use cases where the results of \cite{yuan2019dynamic} do not apply (see Section \ref{sec:ext}). Every dynamic regret rate proposed in this paper are adaptive to $V_T$ in the sense that the algorithms do not require the knowledge of $V_T$ ahead of time.

In the second part of paper, we concern ourselves with a related but slightly different setting, under the static regret framework. More precisely, we provide a lower bound on a certain \emph{penalized regret} (see Definition \ref{def:pen-regret}) for the problem of competing against a \emph{fixed function} in an RKHS under squared error losses. We show (Theorem \ref{thm:lb}) that the penalized regret has a lower bound of $\Omega \left(  \log \left | \bs I + \frac{1}{a} \bs K \right|\right)$ for some fixed $a > 0$ where $|\cdot |$ denotes the determinant.

This establishes the near optimality of online clipped Kernel Ridge Regression (KRR) \citep{KRR} and kernel-AWV \citep{Jzquel2019EfficientOL} thus solving a problem open since the work of \cite{KRR}. The penalized regret we consider is similar to the one studied in \cite{vovk2001} for finite dimensional linear regression.

To summarize, this paper records a preliminary set of results about the dynamic regret of strongly adaptive methods. The findings in this work also initiated the study on minimax optimality of SA methods in a setting where improper learning is allowed \citep{improperDynamic}. Specifically below are the key contributions of this work.

\begin{itemize}
    \item We show that Strongly Adaptive (SA) algorithms are sufficient to guarantee the dynamic regret rates of $\tilde O(\sqrt{TV_T} \vee \log T)$ for strongly convex losses and  $\tilde O(\sqrt{dTV_T} \vee d\log T)$ for exp-concave losses (see Theorems \ref{thm:sc} and  \ref{thm:ec} respectively). Combined with the results of \cite{zhang2018dynamic}, we feature SA methods as a \emph{unifying framework} for \emph{simultaneously} controlling dynamic regret with variationals $V_T$ and $C_T$.
    
    \item We demonstrate the versatility of this perspective by deriving several extensions (Theorems \ref{thm:ec-inv} and \ref{thm:sq-dyn}) where the results of \cite{yuan2019dynamic} don't apply.  In particular, for competing against set of linear predictors that output bounded predictions as in \cite{Luo2016Sketch}, we show that SA methods enjoy dynamic regret rate that is \emph{independent} of the diameter of the decision set. To the best of our knowledge this is the \emph{first} time, dynamic regret rate has been proposed for such a benchmark set which is often more of practical interest than set of linear predictors with bounded $L^2$ norm.
    
    \item We provide a lower bound (Theorem \ref{thm:lb}) that establishes the near optimality of online clipped KRR and kernel-AWV thus solving a problem open since the work of \cite{KRR}.
\end{itemize}

The rest of the paper is organized as follows. In Section \ref{sec:lit}, we discuss the related works followed by a discussion on preliminaries in Section \ref{sec:prelim}. We present the dynamic regret guarantees for strongly convex losses and exp-concave losses in Section \ref{sec:finite}. The extensions to competing against bounded linear predictors and online kernel regression is presented in Section \ref{sec:ext}. The lower bound on penalized regret for online clipped KRR and kernel-AWV is explored in Section \ref{sec:lb}.

\section{Related Work} \label{sec:lit}
The notion of static regret is very common and there are many well known algorithms for controlling it. For example, when the decision set has bounded diameter and the functions have bounded gradients in the Euclidean norm,  Online Gradient Descent (OGD) with appropriately chosen step size can yield a static regret of $O(\sqrt{T})$ and $O(\log T)$ for convex and strongly convex losses respectively. When the losses are exp-concave, Online Newton Step (ONS) attains regret bounds of $O(d \log T)$ \citep{hazan2007logregret}. For controlling static regret in arbitrary norms, one may use Online Mirror Descent (see for eg. \citep{bubeck2015ConvexOA}). Parameter free versions of static regret minimizing algorithms have been proposed in the works of \cite{Orabona2016CoinBA,Cutkosky2018BlackBoxRF}. \cite{Ross2013NormalizedOL, Luo2016Sketch} propose algorithms to control the static regret for competing against bounded linear predictors.

It is well known that to attain sub-linear dynamic regret, one must impose some regularities on the comparator sequence or the loss function sequence.
\cite{zinkevich2003online} shows that OGD can be used to attain $O(\sqrt{T}(1+V_T))$ dynamic regret. This has been improved by \cite{zhang2018adaptive} to $O(\sqrt{T(1+V_T)})$ which is minimax optimal when the loss functions are convex. 

When the losses are strongly convex and if we restrict our comparators to be the sequence of unique minimizers $\theta^*_{t} = \argmin_{\bs x \in \cD} f_t(\bs x)$, one may define a path variational $V_T^* = \sum_{t=2}^{T} \| \theta^*_{t} - \theta^*_{t-1}\|$. \cite{mokhtari2016dynamic} shows that OGD enjoys a dynamic regret of $O(1+V^*_T)$. Though this can be used to upper bound the dynamic regret against any comparators in Eq.\eqref{eq:d-regret}, such an upper bound can be very vacuous. 

The variational in Eq.\eqref{eq:c_t} is introduced by \cite{besbes2015non} and they propose a restarted  OGD procedure to yield dynamic regret of $O(T^{2/3}C_T^{1/3})$ and $O(\sqrt{TC_T \log T})$ for convex and strongly convex losses respectively. However, they require the apriori knowledge of the bound $C_T$ which may not be possible to obtain in practice. \cite{jadbabaie2015online} proposes a unifying strategy that yields dynamic regret bounds that simultaneously depend on $V_T$ and $C_T$ when the only condition on losses is convexity. 


There are a number of works related to non-parametric regression and online learning with RKHS. We only recall here the ones that are perhaps most relevant to our paper. The works of \cite{KRR,Jzquel2019EfficientOL} propose algorithms for kernel regression that control certain penalized regret (see Definition \ref{def:pen-regret}) when learning with squared error losses . A sequential bayesian strategy that controls the penalized regret with log losses has been proposed in \cite{bayesgp}. \cite{Zhang2015DivideAC} establishes lower bounds on static regret for kernel regression with squared error losses, however, the optimality of algorithms in  \cite{KRR,Jzquel2019EfficientOL} in terms of penalized regret is still unknown. The key difference is that in the static regret notion of \cite{Zhang2015DivideAC}, we are competing against a subset of functions with RKHS norms bounded by some \emph{known} radius while in the notion of penalized regret we are competing against the \emph{entire} RKHS (see Section \ref{sec:lb} for more details).
In this paper, we provide a positive result that the algorithms in \cite{KRR,Jzquel2019EfficientOL} are indeed nearly minimax optimal in terms of its penalized regret.

\section{Preliminaries} \label{sec:prelim}
The results in this paper hold for general Strongly Adaptive algorithms, but for concreteness we will phrase them in terms of a particular algorithm called Follow-the-Leading-History (FLH) \citep{hazan2007adaptive}.

\begin{figure}[h!]
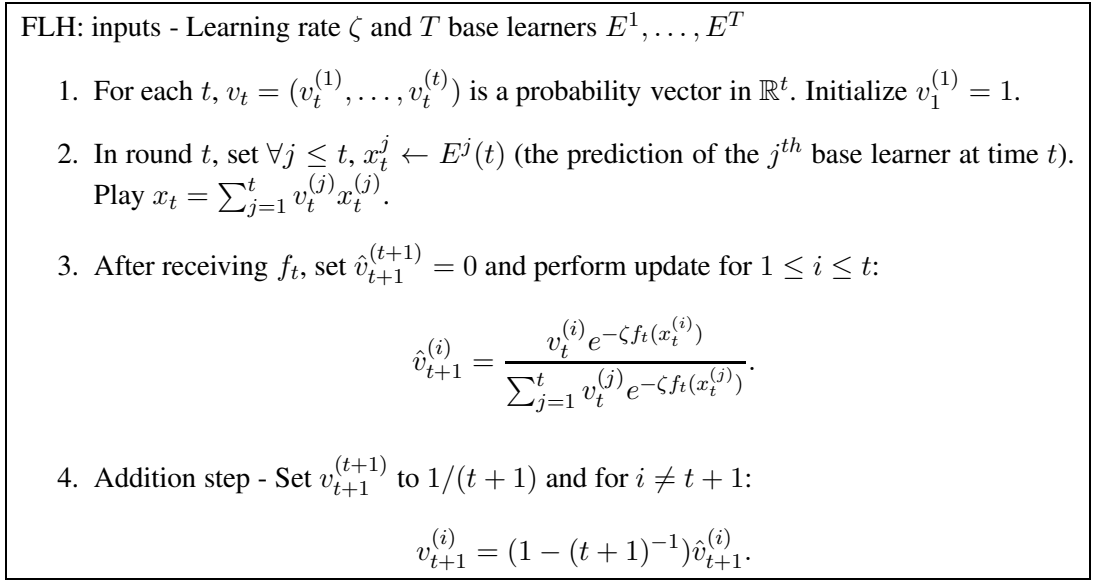

	\centering
	\fbox{
		\begin{minipage}{14 cm}
		FLH: inputs - Learning rate $\zeta$ and $T$ base learners $E^1,\ldots,E^T$
            \begin{enumerate}
                \item For each $t$, $v_t = (v_t^{(1)},\ldots,v_t^{(t)})$ is a probability vector in $\mathbb{R}^t$. Initialize $v_1^{(1)} = 1$.
                \item In round $t$, set $\forall j \le t$, $x_t^j \leftarrow E^j(t)$ (the prediction of the $j^{th}$ base learner at time $t$). Play $x_t =  \sum_{j=1}^t v_t^{(j)}x_t^{(j)}$.
                \item After receiving $f_t$, set $\hat v_{t+1}^{(t+1)} = 0$ and perform update for $1 \le i \le t$:
                \begin{align}
                    \hat v_{t+1}^{(i)}
                    &= \frac{v_t^{(i)}e^{-\zeta f_t(x_t^{(i)})}}{\sum_{j=1}^t v_t^{(j)}e^{-\zeta f_t(x_t^{(j)})}}.
                \end{align}
                \item Addition step - Set $v_{t+1}^{(t+1)}$ to $1/(t+1)$ and for $i \neq t+1$:
                \begin{align}
                    v_{t+1}^{(i)}
                    &= (1-(t+1)^{-1}) \hat v_{t+1}^{(i)}.
                \end{align}
            \end{enumerate}
		\end{minipage}
	}
	\caption{FLH algorithm}
	\label{fig:flh}
\end{figure}
We recall that a function $f_t$ is said to be $H$-strongly convex in the domain in the domain $\mathcal{D}$ if it satisfies 
\begin{equation}
\label{eqn:hsc}
f_t(\bs y) \ge f_t(\bs x) + (\bs y - \bs x)^T \nabla f_t(\bs x) + \frac{H}{2} \|\bs x - \bs y \|^2,
\end{equation}
for all $\bs x, \bs y \in \cD$. Further, $f_t$ is said to be $\alpha$-exp-concave if the last term in Eq.\eqref{eqn:hsc} is replaced by $\frac{\alpha}{2} \left( (\bs y - \bs x)^T \nabla f_t(\bs x) \right)^2.$
 
 


FLH enjoys the following guarantee against any base learner.
\begin{proposition}\label{prop:flh} \citep{hazan2007adaptive}
Suppose the loss functions are exp-concave with parameter $\alpha$. For any interval $I = [r,s]$ in time, the algorithm FLH with learning rate $\zeta = \alpha$ gives $O(\alpha^{-1}( \log r + \log |I|))$ regret against the base learner in hindsight.
\end{proposition}

For the case of exp-concave losses, one can maintain base learners $E^1,\ldots,E^T$ in Fig.\ref{fig:flh} as ONS algorithms that start at time points $1,\ldots,T$. Since each ONS instance achieves an $O(d \log T)$ static regret, Proposition \ref{prop:flh} implies that the corresponding FLH with $\zeta = \alpha$ attains $O(d \log T)$ static regret in any interval. 

Losses that are $H$-strongly convex and $G$-Lipschitz are known to be $O(H/G^2)$ exp-concave \citep{hazan2007logregret}. Further OGD attains $O(\log T)$ static regret. Hence FLH with OGD base learners and $\zeta = H/G^2$ can yield an $O(\log T)$ static regret in any interval when the loss functions are $H$-strongly convex. In Definition \ref{def:sa} if we restrict to minimax optimality wrt to interval length, these observations give rise to the following proposition.

\begin{proposition}
FLH algorithm in Fig.\ref{fig:flh} with base learners as OGD and ONS are Strongly Adaptive when the losses are strongly convex and exp-concave respectively.
\end{proposition}

\section{Dynamic regret for strongly convex and exp-concave losses} \label{sec:finite}
We start by showing that SA methods can serve as a principled way of achieving dynamic regret rates (up to log factors) of \cite{yuan2019dynamic}. We assume that the loss functions are Lipschitz in the decision set.

\begin{assumption} \label{as:lip}
The loss functions $f_t$ satisfy $|f(\bs x) - f(\bs y)| \le G \|\bs x - \bs y \|$ for all $\bs x, \bs y \in \cD$.
\end{assumption}

\subsection{Strongly convex losses}
In this section we derive dynamic regret rates when the loss functions are $H$-strongly convex. We show that by appropriately instantiating the base learners in FLH, one can control the dynamic regret rates. The unspecified proofs are provided in the Appendix.

\begin{theorem} \label{thm:sc}
Suppose the loss function $f_t$ are $H$-strongly convex loss and satisfy Assumption \ref{as:lip}. Running FLH with learning rate $\zeta = H/G^2$ and  base learners as online gradient descent (OGD) with step size $\eta_t = 1/Ht$ results in a dynamic regret of $\tilde O\left (\sqrt{TV_T} \vee 1 \right)$, where $\tilde O(\cdot)$ hides dependence on constants $H,G$ and poly-logarithmic factors of $T$.
\end{theorem}

We start with some useful lemmas for proving this theorem. In Lemma \ref{lem:part}, we divide the time horizon into various bins such that the path variation of the comparator sequence incurred within these bins is at-most a quantity that will be tuned later. In Lemma \ref{lem:main1}, for each bin, we bound the dynamic regret by the sum of static regret against the first comparator point within the bin and a term that captures the drift of the remaining sequence of comparator points from the first point. 
\begin{lemma} \label{lem:part}
Let $\tilde V > 0$ be a constant. There exists a partitioning $\cP$ of the sequence $\bs z_1,\ldots,\bs z_t$ into $M$ bins viz $\{ \{[i_s,i_e]\}_{i=1}^{M}\}$ such that:
\begin{enumerate}
    \item $M := |\cP| = O\left(\max\{ V_T/\tilde V, 1 \} \right)$.
    \item For all $[i_s,i_e] \in \cP$ with $i_s < i_e$, $\sum_{j=i_s+1}^{i_e} \| \bs z_{j} - \bs z_{j-1}\| \le \tilde V$.
\end{enumerate}
\end{lemma}

\begin{lemma} \label{lem:main1}
Assume that the losses are $\alpha$-exp-concave. Let $\bs x_t$ be the predictions made by FLH with learning rate set as $\zeta  = \alpha$. Let $R(L)$ be the static regret incurred by the base learners in an interval of length $L$. Let $\cP$ be the partition of $[T]$ produced in Lemma \ref{lem:part}. Represent each element in the partition by $[i_s,i_e], \: i = 1,\ldots, M$. Then we have,
\begin{align}
    \sum_{t=1}^{T} f_t(\bs x_t) - f_t(\bs z_t)
    &\le  \tilde O \left( \inf_{\tilde V: \frac{V_T}{\tilde V} \ge 1} \left( V_T/\tilde V + GT \tilde V +  \sum_{i=1}^{M} \sum_{t=i_s}^{i_e} R(i_e-i_s+1) \right) \vee R(T)   \right) \label{eq:lem-main}
\end{align}
\end{lemma}

\begin{proof}\textbf{of Theorem} \ref{thm:sc}.
We assume the notations in Lemma \ref{lem:main1}. An $H$-strongly convex loss is $H/G^2$ exp-concave in the decision set $\cD$ \citep{hazan2007logregret}. Further, when the losses are strongly convex, from Theorem 1 of \cite{hazan2007logregret} we have $R(T) = O(\log T)$ for OGD with step size $\eta_t = 1/Ht$. Hence by Lemmas \ref{lem:part} and \ref{lem:main1} we have,
\begin{align}
    \sum_{t=1}^{T} f_t(\bs x_t) - f_t(\bs z_t)
    &\le \tilde O \left(V_T/\tilde V + GT \tilde V +  \sum_{i=1}^{M} \sum_{t=i_s}^{i_e} R(i_e-i_s+1)  \right),\\
    &\le \tilde O \left(V_T/\tilde V + GT \tilde V +  \sum_{i=1}^{M} \log T  \right),\\
    &\le \tilde O \left(V_T/\tilde V + GT \tilde V \right), \label{eq:dyn-sc}
\end{align}
whenever $V_T/\tilde V \ge 1$.

Assume that $V_T \ge 1/T$. In this setting, if we choose $\tilde V = \sqrt{V/T}$, we have $V/\tilde V \ge 1$. Plugging this value to Eq. \eqref{eq:dyn-sc} yields a dynamic regret of $\tilde O(\sqrt{TV})$. 

When $V_T = O(1/T)$, then we have,
\begin{align}
    \sum_{t=1}^{T} f_t(\bs x_t) - f_t(\bs z_t)
    &= \sum_{t=1}^{T} f_t(\bs x_t) - f_t(\bs z_1) + \sum_{t=1}^{T} f_t(\bs z_1) - f_t(\bs z_t),\\
    &\le_{(a)} O(\log T) + GTV_T,\\
    &\le_{(b)} O(\log T),
\end{align}
where line (a) is by strong adaptivity of FLH and Lispchitzness of $f_t$. Line (b) is by the assumption $V_T = O(1/T)$. Combining both cases now yields the theorem.
\end{proof}

\subsection{Exp-concave losses}
In this section, we assume that the losses are $\alpha$-exp-concave and the domain is bounded. Specifically:

\begin{assumption} \label{as:ec}
There exists a constant $D$ such that $\max_{\bs x, \bs y \in \cD} \|\bs x - \bs y \| \le D$.
\end{assumption}

We have the following Theorem.
\begin{theorem} \label{thm:ec}
Suppose the losses $f_t$ are $\alpha$-exp-concave and satisfy Assumptions \ref{as:lip} and \ref{as:ec}. Running FLH with learning rate $\zeta = \alpha$ and ONS as base learners results in a dynamic regret of $\tilde O\left ( \sqrt{dTV} \vee d  \right)$, where $\tilde O(\cdot)$ hides dependence on constants $G,D,\alpha$ and poly-logarithmic factors of $T$.
\end{theorem}

Theorem 2 of \cite{hazan2007logregret}, provides $O(d \log T)$ static regret for ONS under Assumptions \ref{as:lip} and \ref{as:ec}. Theorem \ref{thm:ec} follows by plugging in this static regret guarantee in the arguments of the proof of Theorem \ref{thm:sc}

We conclude this section by two remarks that are applicable to every dynamic regret guarantee presented throughout the paper.

\begin{remark} \label{rem:runtime}
Let $\tau$ be the running time of OGD per round. The FLH procedure incurs a run-time of $O(\tau T)$ per round. This can be improved to $O(\tau \log T)$ by using the AFLH procedure of \cite{hazan2007adaptive} at the cost of increasing the dynamic regret by a logarithmic factor in time horizon $T$.
\end{remark}

\begin{remark}
The FLH procedure doesn't require to know an apriori bound on $V_T$ ahead of time. Hence the dynamic regret in Theorems \ref{thm:sc} and \ref{thm:ec} is adaptive to the variation $V_T$.
\end{remark}

\section{Extensions} \label{sec:ext}
In this section, we demonstrate the versatility of SA methods by deriving new dynamic regret guarantees in various interesting settings.

\subsection{Dynamic regret against bounded linear predictors}
Consider the following learning protocol:

\begin{itemize}
    \item For $t = 1,\ldots,T$:
    \begin{enumerate}
        \item Adversary reveals a feature vector $\bs v_t \in \mathbb{R}^d$.
        \item Learner chooses $\bs w_t \in \mathbb{R}^d$ and predict $\bs w_t^T \bs v_t$.
        \item Adversary reveals a loss $f_t(\bs w) := \ell_t(\bs w^T \bs v_t)$.
        \item Learner suffers loss $\ell_t(\bs w_t^T \bs v_t)$.
    \end{enumerate}
\end{itemize}

Under the above protocol, most of the OCO algorithms typically minimize the regret against a set of benchmark weights (where each weight define a linear predictor) that is bounded in some norm (e.g., the Euclidean $\|\cdot\|_2$ norm). In this section, we follow the path in \cite{Ross2013NormalizedOL,Luo2016Sketch} and study dynamic regret against a set of weights that rather produce bounded predictions. Specifically, define $\mathcal K_t := \{\bs w:  |\bs w^T \bs v_t| \le B\}$. We aim to compete with a benchmark of linear predictors:
\begin{align}
    \mathcal K
    &= \cap_{t=1}^T \mathcal K_t,\\
    &= \{\bs w: \forall t \in [T], |\bs w^T \bs v_t| \le B \},
\end{align}
which basically defines a set of weights that outputs predictions in $[-B,B]$ at the given feature set. As noted in \cite{Luo2016Sketch}, the benchmark set $\mathcal K$ can be much larger than an L2 norm ball. The set $\mathcal K$ is often more useful in practice than a set of weights with bounded norm since it is more easier to choose a reasonable interval of predictions rather than choosing a bound on perhaps non-interpretable norm of the weights. We have the following dynamic regret guarantee.

\begin{theorem} \label{thm:ec-inv}
Suppose the losses $f_t$ are $\alpha$-exp-concave and satisfy Assumption \ref{as:lip}. Further assume that $\ell_t$ are Lipschitz smooth. Running FLH with learning rate $\zeta = \alpha$ and invariant ONS algorithm from \cite{Luo2016Sketch} as base learners results in a dynamic regret of $\tilde O\left ( d\sqrt{TV} \vee d^2  \right)$, where $\tilde O(\cdot)$ hides dependence on constants $G,\alpha$ and poly-logarithmic factors of $T$.
\end{theorem}

Theorem 4 of \cite{Luo2016Sketch}, provides $O(d^2 \log T)$ static regret for a variant of ONS when $\ell_t$ are Lipschitz and $f_t$ are exp-concave. Theorem \ref{thm:ec-inv} follows by plugging in this static regret guarantee in the arguments of the proof of Theorem \ref{thm:sc},

The dynamic regret bounds of \cite{yuan2019dynamic} are derived under the assumption that the norm of the elements in the benchmark set is bounded by some known constant. Specifically, the dynamic regret bound of \cite{yuan2019dynamic} grows as $O(D\sqrt{dTV_T})$ where $D$ is the maximum $L^2$ norm of a predictor in the benchmark set. With benchmark set being $\mathcal K$, this $D$ can be prohibitively large. In this case the \emph{diameter independent} regret guarantee in Theorem \ref{thm:ec-inv} can be much smaller. To the best of our knowledge this is the first time a diameter independent regret guarantee has been proposed for controlling the dynamic regret in terms of $V_T$ when the losses are exp-concave.

\subsection{Dynamic regret for regression against a function space} \label{sec:rkhs}

In this section, we derive dynamic regret guarantees for competing against a sequence of functions in an RKHS induced by the Gaussian kernel. We study a regression setup where the loss is measured using squared errors. Specifically we consider the protocol in Fig.\ref{fig:sq}

\begin{figure}[h!]
	\centering
	\fbox{
		\begin{minipage}{14 cm}
            \begin{enumerate}
            \item For time $t=1,\ldots,T$:
            \begin{enumerate}
                \item Receive $\bs x_t \in \mathbb{R}^d$.
                \item Learner predicts $\hat y_t \in \mathbb{R}$.
                \item Adversary reveals a label $y_t \in [-B,B]$.
                \item Player suffers a loss of $(y_t - \hat y_t)^2$.
            \end{enumerate}
            \end{enumerate}
		\end{minipage}
	}
	\caption{\emph{Interaction protocol with squared error losses.}}
	\label{fig:sq}
\end{figure}

\textbf{Setup and notations.} We represent each function in the RKHS $\cH_k$ by a weight vector $\bs w \in \mathbb{R}^D$ where $D$ can be possible infinite. Let $\bs x \in \mathbb{R}^d$. For a function $f_{\bs w}(\bs x) \in \cH_k$, we have $f_{\bs w}(\bs x) = \bs w^T \bs \phi(\bs x)$ where $\bs \phi(\bs x) \in \mathbb{R}^D$ is the feature embedding of the vector $\bs x$ induced by the Kernel function $k:\mathbb{R}^d \times \mathbb{R}^d \rightarrow \mathbb{R}$. We consider the gaussian kernel where $k(\bs x, \bs y) = \exp(-\| \bs x - \bs y \|^2/(2\sigma^2))$ for some bandwidth parameter $\sigma$. The RKHS norm of the function $f_{\bs w}$ which corresponds to the weight vector $\bs w$ is denoted by $\|\bs w\| := \| f_{\bs w}\|_{\cH_k}$. We denote the determinant of a matrix $\bs A$ by $|\bs A |$.

For a sequence of comparator functions $f_{\bs w_1},\ldots,f_{\bs w_T} \in \cH_{k}$, define the path variational as
\begin{align}
    V_T = \sum_{t=2}^{T} \|f_{\bs w_t} - f_{\bs w_{t-1}} \|_{\cH_k}.
\end{align}

We are interested in controlling the dynamic regret,
\begin{align}
    \sum_{t=1}^{T} (\hat y_t - y_t)^2 - (f_{\bs w_t}(\bs x_t) - y_t)^2,
\end{align}
for a sequence of functions $f_{\bs w_t}$ that belong to the class of functions with bounded RKHS norm defined as $\cD = \{\bs w: \| \bs w\| \le B \}$. For any $\bs w \in \cD$, since $|\bs w^T \phi(\bs x_t)| \le \|\bs w \| \| \phi(\bs x_t)\| = \|\bs w \| \sqrt{k(\bs x_t,\bs x_t)} \le B$ and $|y_t| \le B$, we have that the losses $\ell_t(\bs w) := (\bs w^T \phi(\bs x_t) - y_t)^2$ are $1/(8B^2)$ exp-concave in the domain $\cD$ \citep{hazan2007logregret}. Specifically, for all $\bs u, \bs v \in \cD$, we have

\begin{align}
    \ell_t(\bs v) \ge \ell_t(\bs u) + (\bs v - \bs u)^T \nabla \ell_t(u) + \frac{\alpha}{2} \left( (\bs v - \bs u)^T \nabla \ell_t(u) \right)^2, \label{eq:ec-rkhs}
\end{align}
where $ \alpha = 1/(8B^2)$.

The following theorem (proof deferred to Appendix) controls the dynamic regret in the above prediction framework.

\begin{theorem} \label{thm:sq-dyn}
Assume that the comparator function sequence obeys $\| f_{\bs w_t} \|_{\cH_k} \le B$  and labels obey $|y_t| \le B$ for all $t \in [T]$. Running FLH with learning rate $\zeta = 1/(8 B^2)$ and base learners as PKAWV from \cite{Jzquel2019EfficientOL} with parameter $\lambda = 1$ and basis functions that approximate Gaussian kernel yields a dynamic regret:
\begin{align}
    \sum_{t=1}^{T} (\hat y_t - y_t)^2 - (f_{\bs w_t}(\bs x_t) - y_t)^2
    &\le O \left( (\log T)^{\frac{d+1}{2}} \sqrt{TV_T} \vee (\log T)^{\frac{d+1}{2}}\right).
\end{align}
\end{theorem}

By Theorem 4 of \cite{Jzquel2019EfficientOL}, computational complexity of PKAWV run with the configurations in Theorem \ref{thm:sq-dyn} is $O((\log T)^{2d})$ per round. Hence by Remark \ref{rem:runtime} the runtime of the strategy in Theorem \ref{thm:sq-dyn} is $O(T (\log T)^{2d})$ per iteration, and improves to $O((\log T)^{2d+1})$ via AFLH.

\section{A lower bound for online kernel regression} \label{sec:lb}

In this section we deviate from the framework of the sections above in that we consider penalized regret rather than dynamic regret. In the context of the interaction protocol in Fig. \ref{fig:sq}, various algorithms such as clipped KRR \citep{KRR} are known to provide certain penalized regret guarantees, but it remained an open question whether these guarantees are optimal. In this section we will prove that they are, up to logarithmic terms.

We shall now, for completeness, recall the definition of penalized regret and the penalized regret guarantee of clipped KRR. This guarantee is similar to the one achieved by KAWV algorithm of \citep{Jzquel2019EfficientOL}.
\begin{definition} \label{def:pen-regret}
For any prediction prediction strategy with outputs $\hat y_1,\ldots,\hat y_T$, the penalized regret against an RKHS $\cH_k$, induced by a kernel function $k:\mathbb{R}^d \times \mathbb{R}^d \rightarrow \mathbb{R}$ is defined as
\begin{align}
    R_{T,a} = \sum_{t=1}^T (\hat y_t - y_t)^2 - \inf_{f \in \cH_k} \left( \sum_{t=1}^T (f(\bs x_t) - y_t)^2 + a\| f\|_{\cH_k}^2 \right),
\end{align}
where $a > 0$ is a fixed parameter.
\end{definition}

This notion of regret is a standard metric in online learning dating back to at-least \cite{Herbster2001TrackingTB}. It penalizes the comparator for selecting functions with large RKHS norm.

\begin{proposition} \citep{KRR} \label{prop:krr}
Let $a > 0$ be a constant parameter and let $\hat{y}_t \in \mathbb{R}$ be the predictions of clipped KRR algorithm when run with parameter $a$ and a given kernel function $k:\mathbb{R}^d \times \mathbb{R}^d \rightarrow \mathbb{R}$. Let $\bs K \in \mathbb{R}^{T \times T}$ be the kernel evaluation matrix with $\bs K_{i,j} = k(\bs x_i, \bs x_j)$. Then,
\begin{align}
    R_{T,a}
    &\le 4B^2 \log \left | \bs I + \frac{1}{a}\bs K \right|,
\end{align}
where $\cH_k$ is the RKHS induced by the kernel $k$ and $|\bs A |$ denotes the determinant of a matrix $\bs A$.
\end{proposition}

We remark that regret guarantee of Proposition  \ref{prop:krr}  implies that for any function $f \in \cH_k$, with $\| f\|_{\cH_k} \le C$, clipped KRR guarantees
\begin{align}
    \sum_{t=1}^T (\hat y_t - y_t)^2 - (f(\bs x_t) - y_t)^2
    &\le aC^2 +  4B^2 \log \left | \bs I + \frac{1}{a} \bs K \right|. \label{eq:pen-regret}
\end{align}

It is mentioned in \cite{Jzquel2019EfficientOL} that the regret bound in Eq.\eqref{eq:pen-regret} is minimax optimal for any $f$ with RKHS norm bounded by the some \emph{known} radius $C$, provided the parameter $a$ is chosen by minimizing RHS. However, in the notion of penalized regret, the parameter $a$ is \emph{fixed} and we are competing against the entire RKHS $\cH_k$ with no apriori restrictions on the smoothness $\| f\|_{\cH_k}$. Hence the optimality arguments for competing against functions whose RKHS norm is bounded by a known constant doesn't directly translate to the optimality in terms of penalized regret.

To establish the optimality of clipped KRR or KAWV for competing against the entire RKHS in a fully adversarial setting, one must lower bound their penalized regret. For the case of polynomial kernels we have the feature maps $\phi(\bs x) \in \mathbb{R}^D$, where $\bs x \in \mathbb{R}^d$ with $D = O(d^2)$. So the problem reduces to online linear regression. In this case, the penalized regret with parameter $a = 1$ grows as $O(D \log T)$ \citep{bayesgp} and a nearly matching lower bound for the penalized regret is shown in \cite{vovk2001} (or see Theorem 11.9 in \citep{BianchiBook2006}). In what follows, we develop a lower bound on the penalized regret for arbitrary Mercer kernels that establishes the near optimality (modulo logarithmic factors) of KRR and KAWV for competing against the entire function class $\cH_k$.

\begin{theorem} \label{thm:lb}
Consider the interaction protocol in Fig.\ref{fig:sq}. There exists a choice of $B = \Theta(\sqrt{\log T})$, $\bs x_t \in \mathbb{R}^d$ and $y_t \in [-B,B]$ such that for any algorithm, the penalized regret with parameter $a > 0$ is lower bounded by
\begin{align}
    \log \left | \bs I + \frac{1}{a} \bs K \right| + 2 \log \left( 1 - \frac{1}{T} \right),
\end{align}
for all $T > 1$.
\end{theorem}
Our argument is non-constructive in the sense that we do not explicitly exhibit a sequence of labels satisfying the stated penalized regret lower bound. The proof is facilitated by a reduction to the problem of lower bounding the penalized regret in log loss games. We induce a prior such that a function $f \in \cH_k$ is given the weight of $\exp(-a \| f\|_{\cH_k}^2/2)$. We compute the MAP estimator with this prior and a Gaussian likelihood. The proof is completed by lower bounding the log loss regret of the MAP estimate.
\begin{remark}
Thus the lower bound in Theorem \ref{thm:lb} matches the upper bound in Proposition \ref{prop:krr} up to logarithmic factors in $T$ and a fast decaying additive term of $\log(1-1/T)$. Hence we conclude that the algorithms such as clipped KRR and KAWV are nearly minimax optimal in terms of their penalized regret.
\end{remark}

We start the formal proof of the lower bound with a technical lemma whose proof is in the Appendix.

\begin{lemma} \label{lem:prob-lb}
Let $\cX = \{\bs x_1, \ldots, \bs x_T \}$ be a given ordered set.
Consider the class of functions $\mathcal G = \{\theta(\bs x): \theta(\bs x) = \sum_{i=1}^{T} \alpha_i k(\bs x_i,\bs x) \}$, for a given kernel $k$. Define $\bs \theta_{\bs x_{1:T}} = [\theta(\bs x_1),\ldots,\theta(\bs x_T)]^T$. Fix a likelihood $Q(\bs y | \bs \theta_{\bs x_{1:T}}) = \cN(\bs \theta_{\bs x_{1:T}},\bs I_n)$, where $\bs y = [y_1,\ldots, y_n]^T$. Let $\kappa^2 = \max_{i \in [T]} k(\bs x_i,\bs x_i)$ and $B^2 = 2 (1+\kappa^2/a)\log T$. Then we have,
\begin{align}
    \inf_{P} \sup_{\bs y: \| \bs y\|_\infty \le B} - \log P(\bs y) - \inf_{\theta \in \mathcal G} \{ - \log Q(\bs y | \theta) + \| \theta\|^2_{\cH_k}\}
    &\ge \frac{1}{2} \log \left|\bs I + \frac{1}{a} \bs K \right | + \log \left( 1 - \frac{1}{T} \right),
\end{align}
for all $T > 1$
\end{lemma}

We may now proceed with the proof of Theorem \ref{thm:lb}.
\begin{proof} \textbf{of Theorem \ref{thm:lb}}
Let the adversary fix $\cX = \{ \bs x_1, \ldots, \bs x_t\}$, an ordered set of distinct points in $\mathbb{R}^d$ such that $k(\bs x, \bs x) \le \kappa^2$.

By the representor theorem \citep{wahba1990spline} we know that the infimium in\\ $\inf_{f \in \cH_k} \left( \sum_{t=1}^{T} (f(\bs x_t)-y_t)^2  + a\| f \|_{\cH_k}^2\right)$ is achieved by some function in the class $\mathcal G = \{\theta(\bs x) | \theta(\bs x) = \sum_{t=1}^{T} \bs \alpha_t k(\bs x_t , \bs x), \bs \alpha \in \mathbb{R}^T\}$. Hence inorder to compete with the entire RKHS $\cH_k$ it is sufficient to compete with the function class $\mathcal G$.

For a function $\theta \in \mathcal G$ let $\bs \theta_{\bs x_{1:T}} = [\theta(\bs x_1),\ldots,\theta(\bs x_T)]^T$. Let $\bs y = [y_1,\ldots,y_t]^T$. Define $Q(\bs y | \bs \theta_{\bs x_{1:T}}) = \cN(\bs \theta_{\bs x_{1:T}}, \bs I)$ where $\cN$ is the density of Gaussian distribution. We have

\begin{align}
     \inf_{f \in \cH_k} \left( \sum_{t=1}^{T} (f(\bs x_t)-y_t)^2  + \| f \|_{\cH_k}^2\right)
    &= 2 \cdot \inf_{\theta \in \mathcal G} \left( -\log Q(\bs y | \bs \theta_{\bs x_{1:T}})   + \| \bs \theta \|_{\cH_k}^2\right) - T \log (2 \pi) \label{eq:comp}    
\end{align}

Let $\cP$ be the collection of all joint probability densities on the outcomes. For a given vector $\bs \mu \in \mathbb{R}^T$, define $P_{\bs \mu} = \cN(\bs \mu, \bs I_T)$. Define such a class of joint densities $\tilde{\cP} = \{P_{\bs \mu} :  P_{\bs \mu}(\bs y) = N(\bs \mu, \bs I_T), \bs \mu \in \mathbb{R}^T\}$. Clearly $\tilde{\cP} \subset \cP$. For any $P_{\bs \mu} \in \tilde {\cP}$, we have
\begin{align}
     \sum_{t=1}^{T} (y_t - \mu_t)^2
    &= -2 \cdot \log P_{\bs \mu}(\bs y) - T\log(2 \pi).\label{eq:pred}
\end{align}

Let $B$ be chosen such that $B^2 = 2 \log T (1 + \kappa^2/a)$. Combining \eqref{eq:comp} and \eqref{eq:pred} we can lower bound the penalized regret as,
\begin{align}
    R_n^*
    &= 
    \inf_{\bs \mu} \sup_{\bs y, \| \bs y\|_\infty \le B} \left(\sum_{t=1}^{T} (y_t - \mu_t)^2 - \sup_{f \in \cH_k} \left( \sum_{t=1}^{n} (f(t)-y_t)^2  + \| f \|_{\cH_k}^2\right)\right),\\
    &= 2 \cdot \inf_{P \in \tilde{\cP}} \sup_{\bs y, \| \bs y\|_\infty \le B} \left( -\log P (\bs y) - \inf_{\theta \in \mathcal G} \left( -\log Q(\bs y | \bs \theta_{\bs x_{1:T}})   + \| \bs \theta \|_{\cH_k}^2\right)\right), \\
    &\ge 2 \cdot \inf_{P \in \cP} \sup_{\bs y, \| \bs y\|_\infty \le B} \left( -\log P (\bs y) - \inf_{\theta \in \Theta} \left(-\log Q(\bs y | \bs \theta_{\bs x_{1:T}})   + \| \bs \theta \|_{\cH_k}^2\right)\right)\\
    &\ge  \log \left | \bs I + \frac{1}{a} \bs K \right| + 2 \log \left( 1 - \frac{1}{T} \right),
\end{align}
The last line is due to lemma \ref{lem:prob-lb}.
\end{proof}

\section{Conclusion}
In this work, we derived dynamic regret rates for SA methods when the loss functions have curvature properties such as strong convexity or exp-concavity. Combined with the work of \cite{zhang2018dynamic}, our results indicate that SA methods can be viewed as a unifying framework for controlling dynamic regret in an OCO setting. \cite{yuan2019dynamic} (see Proposition 1 there) establishes minimax optimality of $O(\sqrt{TV_T})$ rate for certain ranges of $V_T$. However it is an open question to establish lower bounds that holds for all values of $V_T$.
In the second part of the paper, we provided a lower bound on penalized regret that establishes the near optimality of various known algorithms for online kernel regression.


\bibliography{db,tf}
\newpage
\appendix
\section{Omitted Proofs}

\begin{proof}\textbf{ of Lemma \ref{lem:part}}
For an interval $[s,e]$ define $V_{s \rightarrow e} = \sum_{t=s+1}^{e} \| \bs z_{t} - \bs z_{t-1}\|$ which is the path variation incurred by the comparator within the interval $[s,e]$. Define $V_{s \rightarrow s} = 0$. Consider the following partitioning scheme.
\begin{enumerate}
    \item Inputs: $\tilde V > 0$, $\bs z_1,\ldots,\bs z_T$.
    \item Initialize $\mathcal P \leftarrow \Phi$ and $s \leftarrow 1$.
    \item For time $t = 1,\ldots,T$:
    \begin{enumerate}
        \item If $V_{s\rightarrow t} \ge \tilde V$:
        \begin{enumerate}
            \item Add $[s,t-1]$ to $\cP$.
            \item Set $s \leftarrow t$.
        \end{enumerate}
    \end{enumerate}
\end{enumerate}

Let $M:=|\cP|$. Let's enumerate the intervals in $\cP$ by $[i_s,i_e]$ with $i=1,\ldots,M$. Assume that $M > 1$. We have,
\begin{align}
    V_T 
    &\ge \sum_{i=1}^{M-1} V_{i_s \rightarrow i_e+1},\\
    &\ge (M-1) \tilde V,
\end{align}
where the last line follows from Steps 3(a,i,ii) of the partitioning scheme. Now rearranging yields the lemma.
\end{proof}

\begin{proof}\textbf{ of Lemma \ref{lem:main1}}
Consider a bin $[i_s,i_e] \in \cP$ where $i \in [M]$. Let $\bs x_t$ be the predictions of the FLH with learning rate $\zeta = \alpha$. Let $\bs p_t$ be the predictions made by the base learner that wakes at time $i_s$. Due to Theorem 3.2 of \cite{hazan2007adaptive}, we have
\begin{align}
    \sum_{t=i_s}^{i_e} f_t(\bs x_t) - f_t(\bs z_t)
    &\le \sum_{t=i_s}^{i_e} f_t(\bs p_t) - f_t(\bs z_t) + O(\log T).
\end{align}

We have,
\begin{align}
    \sum_{t=i_s}^{i_e} f_t(\bs p_t) - f_t(\bs z_t)
    &= \sum_{t=i_s}^{i_e} f_t(\bs p_t) - f_t(\bs z_{i_s}) +\sum_{t=i_s}^{i_e} f_t(\bs z_{i_s}) - f_t(\bs z_t),\\
    &\le_{(a)} R(i_e-i_s+1) + \sum_{t=i_s}^{i_e} f_t(\bs z_{i_s}) - f_t(\bs z_t),\\
    &\le_{(b)} R(i_e-i_s+1) + G \tilde V (i_e-i_s+1),
\end{align}
where line (a) follows from static regret guarantee of the base learner and line (b) is due to Assumption \ref{as:lip} and the fact that $\sum_{j=i_s+1}^{i_e} \| \bs z_{j} - \bs z_{j-1}\| \le \tilde V$. due to the partitioning scheme in Lemma \ref{lem:part}.

Hence summing across all bins yields,
\begin{align}
    \sum_{t=1}^{T} f_t(\bs x_t) - f_t(\bs z_t)
    &\le \tilde O \left(V_T/\tilde V + GT \tilde V +  \sum_{i=1}^{M} \sum_{t=i_s}^{i_e} R(i_e-i_s+1)  \right),
\end{align}
whenever $V_T/\tilde V \ge 1$ due to Lemma \ref{lem:part}. Taking an infimium across such $\tilde V$ and including the static regret case $V_T=0$ now yields the Lemma.
\end{proof}

\begin{proof} \textbf{ of Theorem \ref{thm:sq-dyn}}
By Theorem 4 of \cite{Jzquel2019EfficientOL}, if PKAWV algorithm is run with parameter $\lambda > 0$ and appropriately chosen basis, then we incur a regret:
\begin{align}
    \sum_{t=1}^{T} (\hat y_t - y_t)^2 - (f_{\bs w}(\bs x_t) - y_t)^2
    &\le \lambda \| f_{\bs w} \|_{\cH_k}^2 + \frac{3B^2}{2} \log \left | \bs I + \lambda^{-1} \bs K\right |,
\end{align}
where $\bs K \in \mathbb{R}^T \times \mathbb{R}^T$ is the kernel evaluation matrix with $\bs K_{ij} = k(\bs x_i, \bs x_j)$. When $k(\bs x,\bs x) \le 1$, Lemma 3 of \cite{KONS} implies:
\begin{align}
     \log \left | \bs I + \lambda^{-1} \bs K\right |
     &\le d_{eff}(\lambda) (1 + \log(1 + T/\lambda)),
\end{align}
where the effective dimension is defined as $d_{eff}(\lambda) = \text{Tr}\left (\bs K (\bs K + \lambda \bs I)^{-1}\right)$. It is known from \cite{Altschuler2019MassivelySS} that for Gaussian kernels and covariates $\bs x \in \mathbb{R}^d$  with $k(\bs x, \bs x) \le 1$, we have $d_{eff}(\lambda) = O \left( \left( \log \frac{T}{\lambda} \right)^d \right)$.

Hence whenever the comparator functions have bounded RKHS norm, by choosing $\lambda = 1$ we have the static regret bounded as
\begin{align}
    \sum_{t=1}^{T} (\hat y_t - y_t)^2 - (f_{\bs w}(\bs x_t) - y_t)^2
    &\le O \left( (\log T)^{d+1} \right).
\end{align}

Let $\ell_t(\bs w) = (f_{\bs w}(\bs x_t) - y_t)^2$. Since $\|\nabla \ell_t(\bs w)\| \le 2B^2$, we have that $\ell_t(\bs w)$ is Lipschitz smooth in $\cD$. Now plugging the above static regret bound and $G=2B^2$ into Eq.\eqref{eq:lem-main} and following similar steps as in the proof of Theorem \ref{thm:sc}, we obtain the stated final regret bound.

\end{proof}

\begin{proof} \textbf{ of Lemma \ref{lem:prob-lb}}
Let's consider a perturbed kernel function $\tilde k(\bs x, \bs y) = k(\bs x, \bs y) + \epsilon \mathbb{I}_{\bs x = \bs y}$, where $\mathbb{I}_{\cdot}$ is the indicator function assuming values in $\{ 0,1\}$ and $\epsilon > 0$. Consider the class $\tilde {\mathcal G} = \{\tilde \theta(\bs x) : \sum_{i=1}^T \alpha_i \tilde k(\bs x_i, \bs x) \}$ where $\alpha_i \in \mathbb{R} \: \forall i \in [T]$. Consequently we have $\| \tilde \theta \|_{\tilde{\cH}_k}^2 = \bs \alpha^T \tilde{\bs K} \alpha$ for some coefficient vector $\bs \alpha$ and kernel evaluation matrix $\tilde K$.

Fix a prior $q$ over $\tilde {\mathcal G}$ such that $q( \tilde \theta)$ is proportional to $\exp(- a\| \tilde \theta_t\|_{\tilde{\cH}_k^2}/2)$.\\
Let $\tilde{\bs \theta}_{\bs x_{1:T}} = [\tilde \theta(\bs x_1),\ldots,\tilde \theta(\bs x_T)]^T$. A key observation is that the infimium in \\$\inf_{\tilde \theta \in \tilde {\mathcal G}} \{ - \log Q(\bs y | \tilde {\bs \theta}_{\bs x_{1:T}}) + a\| \tilde \theta\|^2_{\tilde{\cH}_k}\}$ is attained at the MAP estimator with likelihood $Q$ and prior $q$.  Let $\tilde{\theta}^{\text{map}}(\bs x) = \sum_{i=1}^T \bs \alpha^{\text{map}}_i \tilde k(\bs x_i, \bs x)$ for some optimal coefficient vector $\bs \alpha^{\text{map}} \in \mathbb{R}^T$.

We have
\begin{align}
   \bs \alpha^{\text{map}} 
   &= \argmin_{\bs \alpha \in \mathbb{R}^T}\|\bs y - \tilde{\bs K} \alpha \|_2^2 + a\bs \alpha^T \tilde{\bs K} \bs \alpha, \label{eq:min}
\end{align}
From first order optimality conditions, $\bs \alpha^{\text{map}}  = (\tilde {\bs K} + a \bs I)^{-1} \bs y$. Consequently we have 
\begin{align}
\tilde \theta^{\text{map}}(t) = \bs y^T (\tilde {\bs K} + a \bs I)^{-1} \tilde {\bs k}(t), \label{eq:theta-map}
\end{align}
where $\tilde{\bs k}(t) = [\tilde k(1,t),\ldots, \tilde k(T,t)]^T$. Therefore,
\begin{align}
    Q(\bs y |\tilde{\bs \theta}_{\bs x_{1:T}}^{\text{map}})
    &= \cN(\bs y | \tilde {\bs K} (\tilde {\bs K} + a\bs I)^{-1} \bs y, \bs I_n),\\
    &= \frac{1}{(2\pi)^{T/2}} \exp(-\frac{a^2}{2}\bs y^T (\tilde {\bs K} + a \bs I)^{-2} \bs y), \label{eq:likelyhood}
\end{align}
where the last line is due to the fact $\bs I - \tilde {\bs K}( \tilde {\bs K} + \bs aI)^{-1} = ( \tilde {\bs K} + \bs aI)( \tilde {\bs K} + \bs aI)^{-1} - \tilde {\bs K}( \tilde {\bs K} + \bs aI)^{-1} = a ( \tilde {\bs K} + \bs aI)^{-1}$.

Since $\| \tilde \theta^{\text{ map}}(\bs y)\|_{\tilde{\cH}_k}^2 = \bs y^T (\tilde {\bs K} + a \bs I)^{-1} \tilde{\bs K} (\tilde {\bs K} + a \bs I)^{-1} \bs y$ we have
\begin{align}
    q(\tilde \theta^{\text{map}})
    &= \frac{1}{(2\pi)^{T/2}} \sqrt{a |\bs K|} \exp(-\frac{a}{2} \bs y^T (\tilde{\bs K} + a\bs I)^{-1} \tilde{\bs K} (\tilde{\bs K} + a\bs I)^{-1} \bs y ). \label{eq:qmap}
\end{align}

The offline optimal prediction strategy is then given by
\begin{align}
    P^{\text{map}}(\bs y) = \frac{ Q(\bs y |\tilde{\bs \theta}_{\bs x_{1:T}}^{\text{map}}) q(\tilde \theta^{\text{map}})}{\int_{\| \bs y\|_\infty \le B} Q(\bs y |\tilde{\bs \theta}_{\bs x_{1:T}}^{\text{map}}) q(\tilde \theta^{\text{map}}) d\bs y }, \label{eq:pmap}
\end{align}
where we emphasize that $\tilde \theta^{\text{map}}$ depends on $\bs y$ via Eq.\eqref{eq:theta-map}

From Eq.\eqref{eq:qmap}, we have 
\begin{align}
    a\|\tilde \theta^{\text{map}}\|_{\tilde{\cH}_k}^2
    &= - \log(q(\tilde \theta^{\text{map}})) + \frac{1}{2} | a\tilde{\bs K}| - \frac{T}{2} \log(2 \pi). \label{eq:e1}
\end{align}

Let $R := \inf_{P} \sup_{\bs y: \| \bs y\|_\infty \le B}\left( - \log P(\bs y) - \inf_{\theta \in \tilde {\mathcal G}} \{ - \log Q(\bs y | \bs \theta_{\bs x_{1:T}}) + a\| \theta\|^2_{\cH_k}\} \right)$. From Eq.\eqref{eq:pmap} and \eqref{eq:e1}, we have
\begin{align}
    R
    &= - \log P^{\text{map}}(\bs y)- \left( - \log Q(\bs y |\tilde{\bs \theta}_{\bs x_{1:T}}^{\text{map}}) + a\|\tilde \theta^{\text{map}}\|^2_{\tilde{\cH}_k}\right),\\
    &= \int_{\| \bs y\|_\infty \le B} Q(\bs y |\tilde{\bs \theta}_{\bs x_{1:T}}^{\text{map}}) q(\tilde \theta^{\text{map}}) d\bs y - \frac{1}{2} | a\tilde{\bs K}| + \frac{T}{2} \log(2 \pi). \label{eq:e2}
\end{align}

We proceed to evaluate the integral in the above equation. Multiplying Eq.\eqref{eq:likelyhood} and \eqref{eq:qmap} we get
\begin{align}
    Q(\bs y |\tilde{\bs \theta}_{\bs x_{1:T}}^{\text{map}})q(\tilde \theta^{\text{map}})
    &= \frac{1}{(2\pi)^T} \sqrt{|a \tilde{\bs K}|} \exp \left(-\frac{1}{2} \bs y^T \left(\bs I + \frac{1}{a} \tilde{\bs K} \right)^{-1} \bs y \right),\\
    &= \frac{1}{(2\pi)^{T/2}} \sqrt{|a \tilde{\bs K}|} \sqrt{\left|\bs I + \frac{1}{a} \tilde{\bs K} \right|} \cdot \cN \left(\bs 0, \bs I + \frac{1}{a} \tilde{\bs K} \right).
\end{align}

Hence,
\begin{align}
    \int_{\| \bs y\|_\infty \le B} Q(\bs y |\tilde{\bs \theta}_{\bs x_{1:T}}^{\text{map}}) q(\tilde \theta^{\text{map}}) d\bs y
    &= \frac{1}{2} \log \left( \left|\bs I + \frac{1}{a} \tilde{\bs K} \right| \right) + \frac{1}{2}\sqrt{|a \tilde{\bs K}|} - \frac{T}{2} \log(2 \pi) \\
    &\qquad + \log \left( P \left( |z_i | \le B \: \forall i \in [T] \right) \right),
\end{align}
where $\bs z \sim \cN \left(\bs 0, \bs I + \frac{1}{a} \tilde{\bs K} \right)$. Since $k(\bs x, \bs x) \le \kappa^2$ and $B^2 = 2 \log T(1 + \kappa^2/a)$, applying sub-gaussian tail inequality implies

\begin{align}
    P(|z_i| \ge B)
    &\le 2 \exp \left( \frac{-2B^2 \log T}{2 B^2}\right),\\
    &= \frac{1}{T^2},
\end{align}
for a fixed $i \in [T]$. Applying a union bound across all $i$ yields
\begin{align}
    P \left( |z_i | \le B \: \forall i \in [T] \right)
    &\ge 1 - \frac{1}{T}
\end{align}

Substituting the value of the integral into Eq.\eqref{eq:e2} we get
\begin{align}
    R
    &=\frac{1}{2} \log \left( \left|\bs I + \frac{1}{a} \tilde{\bs K} \right| \right) + \log (1 - 1/T).
\end{align}

The proof of the lemma is now complete by noting that $R$ is continuous in $\tilde{\bs K}$ and taking $\epsilon \rightarrow 0^+$.

\end{proof}
\end{document}